\documentclass[runningheads]{llncs}

\usepackage{eccv}

\usepackage{float}

\usepackage{eccvabbrv}

\usepackage{graphicx}
\usepackage{booktabs}
\usepackage{wrapfig}
\usepackage[accsupp]{axessibility}  

\usepackage{enumitem}

\usepackage{hyperref}

\usepackage{orcidlink}

\begin{document}

\title{Author Guidelines for ECCV Submission} 

\title{InjectFlow: Weak Guides Strong via Orthogonal Injection for Flow Matching}
\author{Dayu Wang\inst{1} \and
Jiaye Yang\inst{1} \and
Weikang Li\inst{2} \and
Jiahui Liang\inst{1} \and
Yang Li\inst{1}}

\authorrunning{D.~Wang et al.}

\institute{Baidu Inc. \\
\email{2100010872@stu.pku.edu.cn, yamseyoung@gmail.com, \\ \{liangjiahui03, liyang164\}@baidu.com} \and
Peking University \\
\email{wavejkd@pku.edu.cn}}

\maketitle

\begin{abstract}
Flow Matching (FM) has recently emerged as a leading approach for high-fidelity visual generation, offering a robust continuous-time alternative to ordinary differential equation (ODE) based models. However, despite their success, FM models are highly sensitive to dataset biases, which cause severe semantic degradation when generating out-of-distribution or minority-class samples. In this paper, we provide a rigorous mathematical formalization of the ``Bias Manifold'' within the FM framework. We identify that this performance drop is driven by conditional expectation smoothing, a mechanism that inevitably leads to trajectory lock-in during inference. To resolve this, we introduce \textbf{InjectFlow}, a novel, training-free method by injecting orthogonal semantics during the initial velocity field computation, without requiring any changes to the random seeds. This design effectively prevents the latent drift toward majority modes while maintaining high generative quality. Extensive experiments demonstrate the effectiveness of our approach. Notably, on the GenEval dataset, InjectFlow successfully fixes 75\% of the prompts that standard flow matching models fail to generate correctly. Ultimately, our theoretical analysis and algorithm provide a ready-to-use solution for building more fair and robust visual foundation models.


\end{abstract}

\section{Introduction}
In recent years, the landscape of artificial intelligence has witnessed a paradigm shift towards generative modeling, fundamentally altering the capabilities of automated systems to synthesize complex, high-dimensional data such as photorealistic images, audio, and video\cite{ho2020denoising,rombach2022high,blattmann2023align,kreuk2022audiogen}. At the mathematical core of systems like Stable Diffusion 3 \cite{podell2023sdxlimprovinglatentdiffusion} and state-of-the-art video generation models \cite{blattmann2023stable, bar2024lumiere,brooks2024video} lies the simulation of ordinary and stochastic differential equations (ODEs/SDEs)\cite{esser2024scaling}. While early breakthroughs relied heavily on score-based denoising diffusion probabilistic models (DDPMs)\cite{ho2020denoising, song2020score}, the field is rapidly converging upon Flow Matching as a more generalized, simulation-free approach for training Continuous Normalizing Flows (CNFs) at an unprecedented scale \cite{lipman2022flow,albergo2022building,esser2024scaling,blackforest2024flux}. Flow Matching operates by regressing vector fields of fixed conditional probability paths, transitioning away from the traditional Markov-chain or SDE perspectives toward a purely deterministic mapping\cite{albergo2022building}. This deterministic ODE simulation subsumes hyper-parameters and parametric scores into a unified parametric dynamics framework, governing the continuous evolution from noise to data\cite{karras2022elucidating,lu2022dpm,chen2018neural}.

However, this shift toward deterministic ODEs exposes a critical vulnerability when models are trained on real-world datasets characterized by severe class imbalances and spurious correlations. In classification and non-autoregressive discrete models, such biases manifest as "probability collapse," wherein the predictive distribution sharply concentrates its mass on the majority class or the spurious correlation, ignoring intrinsic, causal features\cite{song2025reweighted,seshadri2024bias,yang2025escaping}. In the visual generation domain, particularly within Flow Matching, this discrete collapse translates into a continuous geometric phenomenon. The analysis indicates that the generative trajectory becomes trapped within a localized subspace defined by the training data's dominant biases\cite{chen2025geometric}.

Notably, similar collapse phenomena have also been observed in large language models (LLMs), where the model’s predictions tend to concentrate on dominant patterns present in the training data. Prior work has shown that injecting orthogonal semantic signals can effectively mitigate such collapse\cite{wang2026studentguidesteacherweaktostrong,wang2026alignmentexpandingreasoningcapacity}. Building upon this insight, we extend the orthogonal injection strategy to the Flow Matching framework, demonstrating that it can likewise alleviate trajectory collapse in generative flows.As illustrated in Figure~\ref{fig:qualitative_comparison}, this approach effectively rectifies the erroneous generation trajectories commonly observed in standard Stable Diffusion models.

Our approach dynamically corrects the generation path through three primary stages:

\begin{figure}[h]
    \centering
    \includegraphics[width=0.9\linewidth]{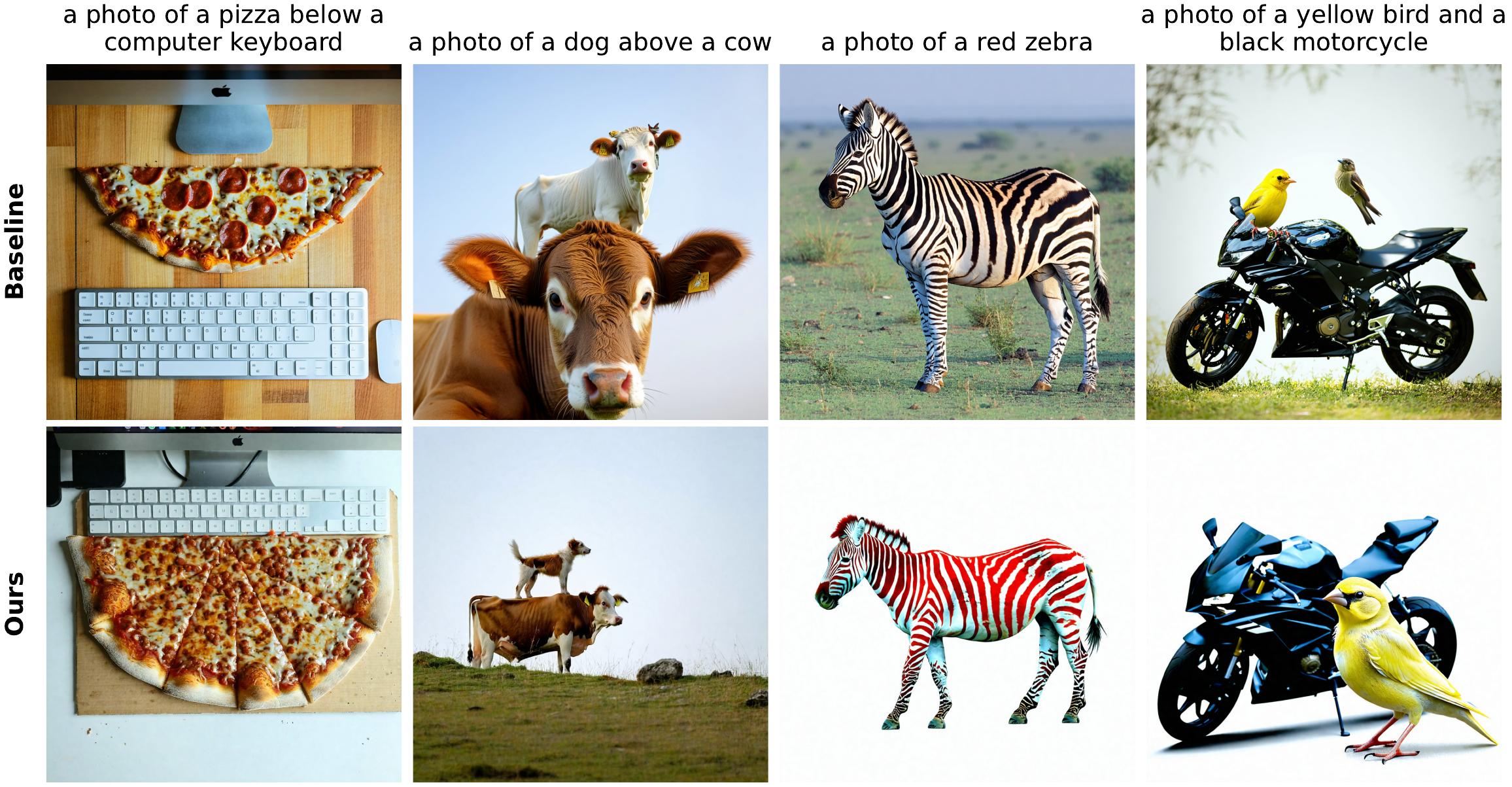}
    \caption{\textbf{Qualitative comparison of generation results.} While the baseline model often succumbs to trajectory collapse, failing to capture specific attribute bindings or complex spatial relationships, our proposed method effectively rectifies these erroneous trajectories, ensuring accurate semantic alignment with the given prompts.}    \label{fig:qualitative_comparison}
\end{figure}

\begin{itemize}
    \item \textbf{Bias Manifold Identification:} We first simulate initial generation trajectories via standard deterministic Ordinary Differential Equation (ODE) solvers. By systematically evaluating and filtering out the erroneous or collapsed paths, we explicitly delineate the underlying bias manifold that traps the model.
    
    \item \textbf{Orthogonal Semantic Injection:} Subsequently, we leverage a weaker, auxiliary model coupled with a dynamic prompt-masking strategy to conduct orthogonal sampling. This process computes a velocity component that is geometrically orthogonal to the primary model's biased vector field, extracting the optimal orthogonal semantic direction for injection.
    
    \item \textbf{SDE-Driven Trajectory Fission:} Finally, conditioned on this orthogonal injection, we transition the sampling process from a deterministic ODE to a Stochastic Differential Equation (SDE). This SDE-driven exploration forces the trajectory to break out of the localized bias subspace, yielding completely novel, diverse, and unbiased generation results.
\end{itemize}

In summary, our key contributions are as follows:

\begin{itemize}
    \item \textbf{Identification of Geometric Collapse:} We explicitly characterize the trajectory collapse phenomenon in continuous normalizing flows, demonstrating how deterministic ODE simulations become confined within localized bias manifolds due to spurious correlations in the training data.
    
    \item \textbf{Weak-to-Strong Orthogonal Sampling:} We introduce a novel test-time intervention that utilizes a weaker companion model and prompt-masking to calculate and inject orthogonal semantic vectors, successfully steering the generation process away from dominant biases without requiring model fine-tuning.
    
    \item \textbf{Hybrid ODE-to-SDE Exploration:} We propose a unified sampling pipeline that seamlessly interrupts the deterministic ODE path with an orthogonal injection, followed by SDE sampling. This strategic shift broadens trajectory exploration, significantly enhancing both the diversity and fidelity of the generated outputs.
\end{itemize}

\section{Related Work}
\textbf{Continuous-Time Generative Models.} The landscape of generative modeling has recently transitioned from stochastic score-based diffusion models \cite{ho2022classifier} to continuous-time deterministic frameworks. Flow Matching \cite{lipman2022flow} introduces a simulation-free approach for training Continuous Normalizing Flows (CNFs) using Optimal Transport paths, significantly improving inference efficiency. Concurrently, Rectified Flows \cite{liu2022flow} iteratively straighten probability paths to enable high-quality generation with minimal integration steps, while Stochastic Interpolants \cite{albergo2022building} provide a unified mathematical framework bridging flows and diffusions. These foundational continuous-time models have been successfully scaled using Multimodal Diffusion Transformers (MMDiT) to achieve state-of-the-art high-resolution image synthesis \cite{esser2024scaling}.

\textbf{Probability Collapse and Trajectory Lock-in.} Despite their computational efficiency, deterministic flow models suffer from continuous-time probability collapse and trajectory lock-in, where models ignore minority semantics and regress into a low-rank bias manifold, a failure mode systematically evaluated by object-focused frameworks like GenEval \cite{ghosh2023geneval}. This lock-in phenomenon, analogous to the cascading hallucinations observed in discrete masked diffusion models \cite{saini2026tabes}, traps the generative ODE trajectory early in the integration process.

\textbf{Inference-Time Interventions.} To bypass this topological lock-in, recent literature focuses on inference-time scaling and orthogonal interventions. SDE Fission and Variance Preserving interpolants reintroduce necessary stochasticity, enabling particle sampling and test-time compute exploration for flow models \cite{kim2025inference}. This stochastic conversion natively supports online reinforcement learning techniques, such as Flow-GRPO, which drastically improves compositional alignment without reward hacking \cite{liu2025flow}. Finally, orthogonal semantic injection leverages heterogeneous weak-to-strong generalization paradigms \cite{burns2023weak} to compute and inject orthogonal semantic residuals at the initial boundary condition, forcefully expanding the tangent space rank and steering the trajectory away from the bias manifold.

\section{Motivation}
\label{sec:motivation}

Recent advancements have established Flow Matching (FM) as the premier paradigm for high-fidelity visual generation. However, despite their empirical success, flow-based models remain highly susceptible to dataset biases, resulting in degraded generation of out-of-distribution or minority-class semantics. 

To understand why models inherently collapse when faced with complex compositional prompts, we must examine the regression objective. Under the Minimum Mean Square Error (MMSE) criterion, the optimal parametric vector field $v_{\theta}^{*}$ learned by FM is the conditional expectation of the target velocity field:
\begin{equation}
v_{\theta}^{*}(x_{t},t,c)=\mathbb{E}_{q(x_{1}|x_{t},c)}[u_{t}(x_{t}|x_{0},x_{1})]
\end{equation}
The target posterior distribution is extremely sparse for complex compositional prompts, therefore the expectation operator acts as a low-pass filter. This mechanism, 
averages out high-frequency complex features and regresses the velocity field entirely into a low-rank subspace spanned by the training set's dominant biases, formally defined herein as the Bias Manifold $\mathcal{M}_{t}$.

To empirically validate this geometric collapse, we designed a novel inference-time latent volume tracking experiment. For a given complex prompt, we initialize 8 independent latent trajectories ($N=8$) driven by the same deterministic ODE. To quantify the dimensional richness of the local manifold, we compute the 7-dimensional volume of the parallelotope spanned by these 8 latent states at each integration step. By designating the 8th trajectory as the local origin, we construct the directional matrix $V \in \mathbb{R}^{7 \times D}$ and strictly compute the volume via the Gram Matrix determinant:
\begin{equation}
Vol = \sqrt{\det(V V^T)}
\end{equation}

\begin{figure}
    \centering
    \includegraphics[width=0.75\linewidth]{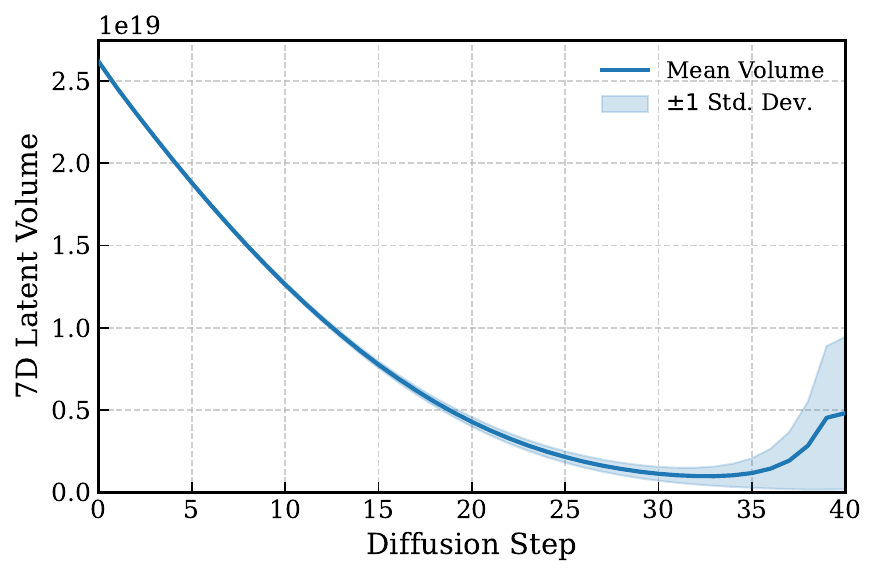}
    \caption{Empirical observation of latent volume collapse across ODE integration steps.}
    \label{fig:volume_eccv}
\end{figure}
The experimental tracking reveals a severe and premature collapse of the latent volume during the earliest integration steps, as shown in Table \ref{fig:volume_eccv}. This empirical observation confirms the phenomenon of Trajectory Lock-in. As Optimal Transport (OT) Flow Matching is explicitly designed to learn straight paths, the initial velocity field $v_{\theta}(x_{0},0,c)$ heavily dictates the global tangent direction. 

Subsequent ODE integration steps merely render details within this erroneous bias manifold, and standard interventions at later steps cannot recover the missing objects because the trajectory has already lost all orthogonal dimensions. This geometric trap fundamentally motivates our proposed approach: rescuing the generation process necessitates a forced Initial Condition Perturbation. By injecting orthogonal semantics during the initial velocity computation, we can mathematically guarantee the trajectory's divergence from the topological basin of attraction, bypassing the probability collapse.

\subsection{Step 1: Bias Manifold Identification}

To understand why Flow Matching models inherently collapse into the Bias Manifold $\mathcal{M}_t$ when faced with complex compositional prompts, we examine the regression objective. Under the Minimum Mean Square Error (MMSE) criterion, the optimal parametric vector field $v_\theta^*$ learned by Flow Matching is the conditional expectation of the target velocity field:
$$v_\theta^*(x_t, t, c) = \mathbb{E}_{q(x_1 | x_t, c)} \left[ u_t(x_t | x_0, x_1) \right]$$
Because the target posterior distribution is extremely sparse for complex compositional prompts, the expectation operator acts as a low-pass filter. It averages out high-frequency complex features and regresses the velocity field entirely into a low-rank subspace spanned by the training set's dominant biases. To explicitly delineate this underlying Bias Manifold $\mathcal{M}_t$ that traps the model, we simulate initial generation trajectories via standard deterministic Ordinary Differential Equation (ODE) solvers. By systematically evaluating and filtering out the erroneous or collapsed paths, we identify the localized subspace driving the geometric collapse.

\subsection{Step 2: Orthogonal Semantic Injection}

The generative trajectory integration in Optimal Transport (OT) Flow Matching is explicitly designed to learn straight paths, where the initial velocity field $v_\theta(x_0, 0, c)$ heavily dictates the global tangent direction. As established, when the conditional expectation forces this initial tangent into the Bias Manifold, we observe the phenomenon of Trajectory Lock-in. Subsequent ODE integration steps merely render details within this erroneous bias manifold, as the trajectory has already lost all orthogonal dimensions. 

As shown in Figure \ref{fig:method_ECCV}, to rescue the generation process, we leverage a weaker, auxiliary model coupled with a dynamic prompt-masking strategy to conduct orthogonal sampling. This process computes a velocity component that is geometrically orthogonal to the primary model's biased vector field, extracting the optimal orthogonal semantic direction for injection and effectively steering the generation process away from dominant biases.

\begin{wrapfigure}{r}{0.5\textwidth}
    \centering
    \includegraphics[width=0.48\textwidth]{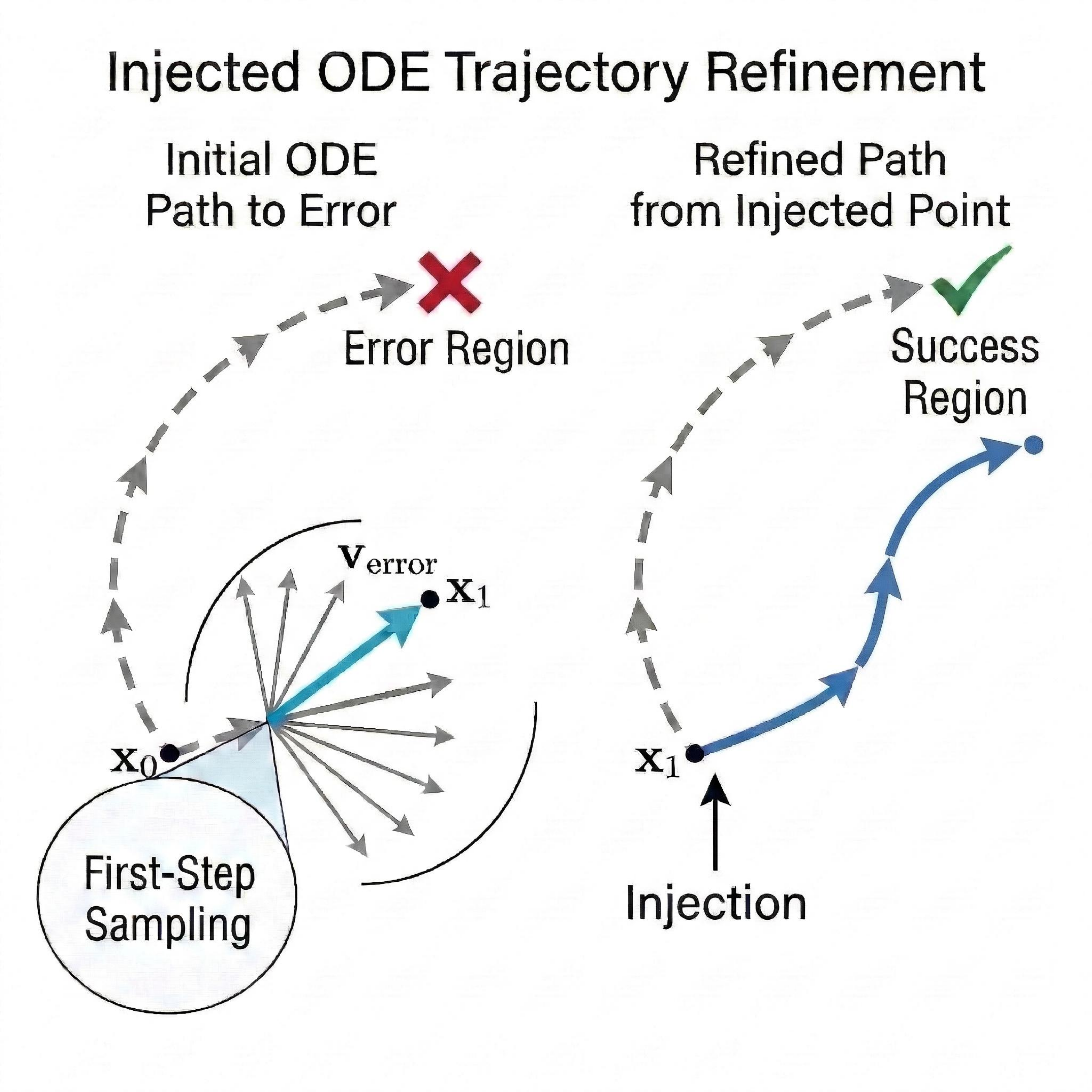} 
    \caption{Our Method}
    \label{fig:method_ECCV}
\end{wrapfigure}

\subsection{Step 3: SDE-Driven Trajectory Fission}

To bypass this geometric trap, rescuing the generation process necessitates a forced Initial Condition Perturbation. During the initial integration steps, we inject the extracted orthogonal semantic direction into the primary model's vector field:
$$v_{new} = v_{primary} + \alpha v_{\perp}$$
where $v_{primary}$ represents the primary model's biased vector field and $v_{\perp}$ is the geometrically orthogonal velocity component extracted via the auxiliary model. 

Finally, conditioned on this orthogonal injection, we transition the sampling process from a deterministic ODE to a Stochastic Differential Equation (SDE). This SDE-driven exploration forces the trajectory to break out of the localized bias subspace. By seamlessly interrupting the deterministic ODE path with an orthogonal injection followed by SDE sampling, we mathematically guarantee the trajectory's divergence from the topological basin of attraction, yielding completely novel, diverse, and unbiased generation results.

\section{Experiments}
\subsection{Settings}

\textbf{Data.} GenEval \cite{ghosh2023geneval} is used in our experiments.

\textbf{Models.} To validate our theoretical framework, we evaluated our method on standard compositional generation benchmarks using the Stable Diffusion 3.5 architecture\cite{esser2024scaling}. Under the weak-strong setting, we employed Stable Diffusion 3.5 Large (8 billion parameters) as the primary model for the base velocity field, and Stable Diffusion 3.5 Medium (2.5 billion parameters) as the weak model to compute the orthogonal residual.

\textbf{Implementation Details.} Our pipeline consists of the following steps: (i) The forward pass of the MMDiT backbone is dynamically intercepted at the initial high-noise step ($start\_step = 1$). (ii) When performing orthogonal injection, we mask specific token indices in the pooled projections to calculate $v_{\perp}$, injecting it across the conditional batch with a scaling factor of $\alpha = 1.0$. (iii) After the Step 1 deterministic injection, we apply Flow-GRPO\cite{liu2025flow} inspired SDE fission for the remaining integration steps, utilizing a noise level of $a = 0.7$ to explore the newly unlocked high-dimensional space.


\section{Main Results}
We evaluate the capability of our method to correct erroneous samples generated by the baseline ODE method on the GenEval benchmark. As detailed in Table \ref{tab:geneval_correction}, our approach demonstrates strong performance across multiple compositional tasks. Most notably, it achieves a 100\% correction rate for the Two Objects task and successfully fixes 81.82\% and 75.00\% of the baseline's failures in the Counting and Color \& Attr tasks, respectively.
\begin{table}[H]
  \centering
  \setlength{\tabcolsep}{6pt}
  \caption{Correction success rate of our method on the GenEval benchmark compared to baseline ODE failures.}
  \label{tab:geneval_correction}
  \begin{tabular}{lccc}
    \toprule
    \textbf{Task} & \textbf{ODE Errors} & \textbf{Ours Corrected} & \textbf{Correction Rate (\%)} \\
    \midrule
    Position      & 80  & 34 & 42.50 \\
    Color \& Attr & 44  & 33 & 75.00 \\
    Two Objects   & 7   & 7  & 100.00 \\
    Colors        & 12  & 7  & 58.33 \\
    Counting      & 22  & 18 & 81.82 \\
    \bottomrule
  \end{tabular}
\end{table}

Qualitative comparisons further validate the efficacy of our orthogonal semantic injection. Figure \ref{fig:result_color_attr} illustrates how our method successfully disentangles and binds the correct colors to specific objects, resolving the semantic bleeding seen in the baseline. Furthermore, Figure \ref{fig:result_count} demonstrates our model's ability to generate the exact number of requested items, overcoming the baseline's tendency to drop or hallucinate entities. Finally, Figure \ref{fig:result_position} highlights our method's precision in adhering to complex spatial positioning constraints.

\begin{figure}[H]
    \centering
    \includegraphics[width=0.7\linewidth]{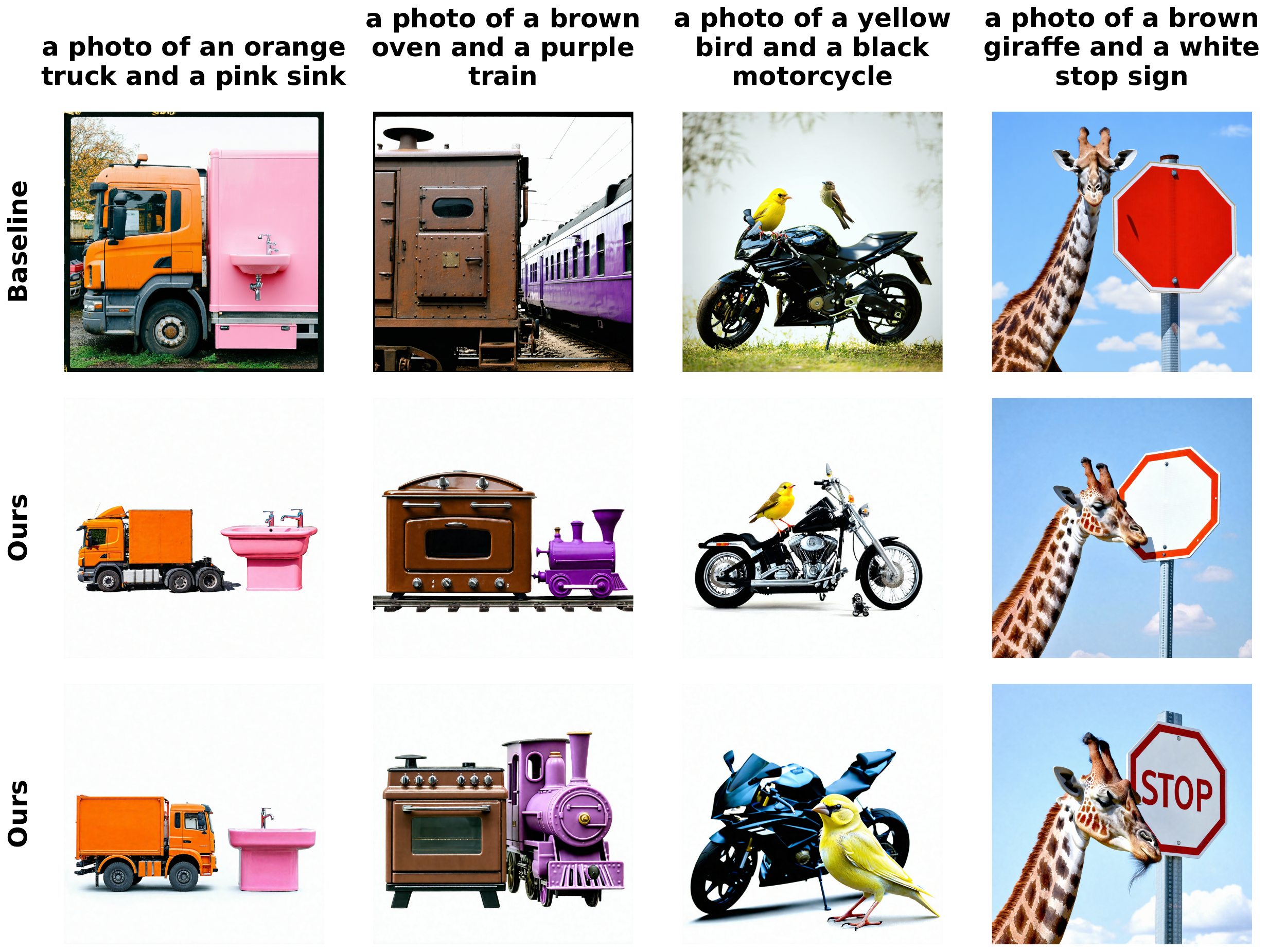}
    \caption{Qualitative comparison showing our method successfully correcting attribute binding and color bleeding issues.}
    \label{fig:result_color_attr}
\end{figure}

\begin{figure}[H]
    \centering
    \includegraphics[width=0.7\linewidth]{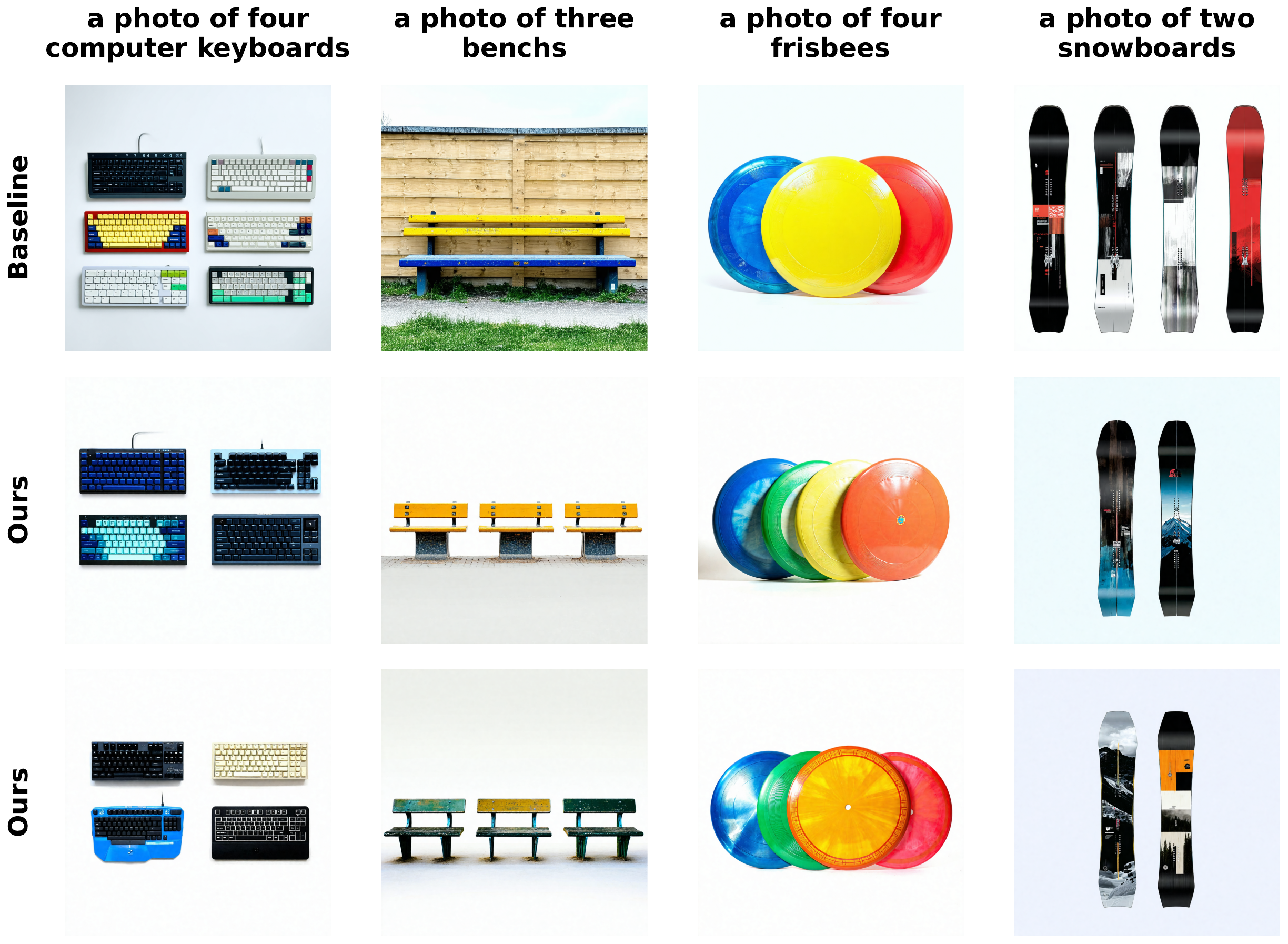}
    \caption{Qualitative comparison demonstrating accurate numerical generation for the counting task.}
    \label{fig:result_count}
\end{figure}

\begin{figure}[H]
    \centering
    \includegraphics[width=0.7\linewidth]{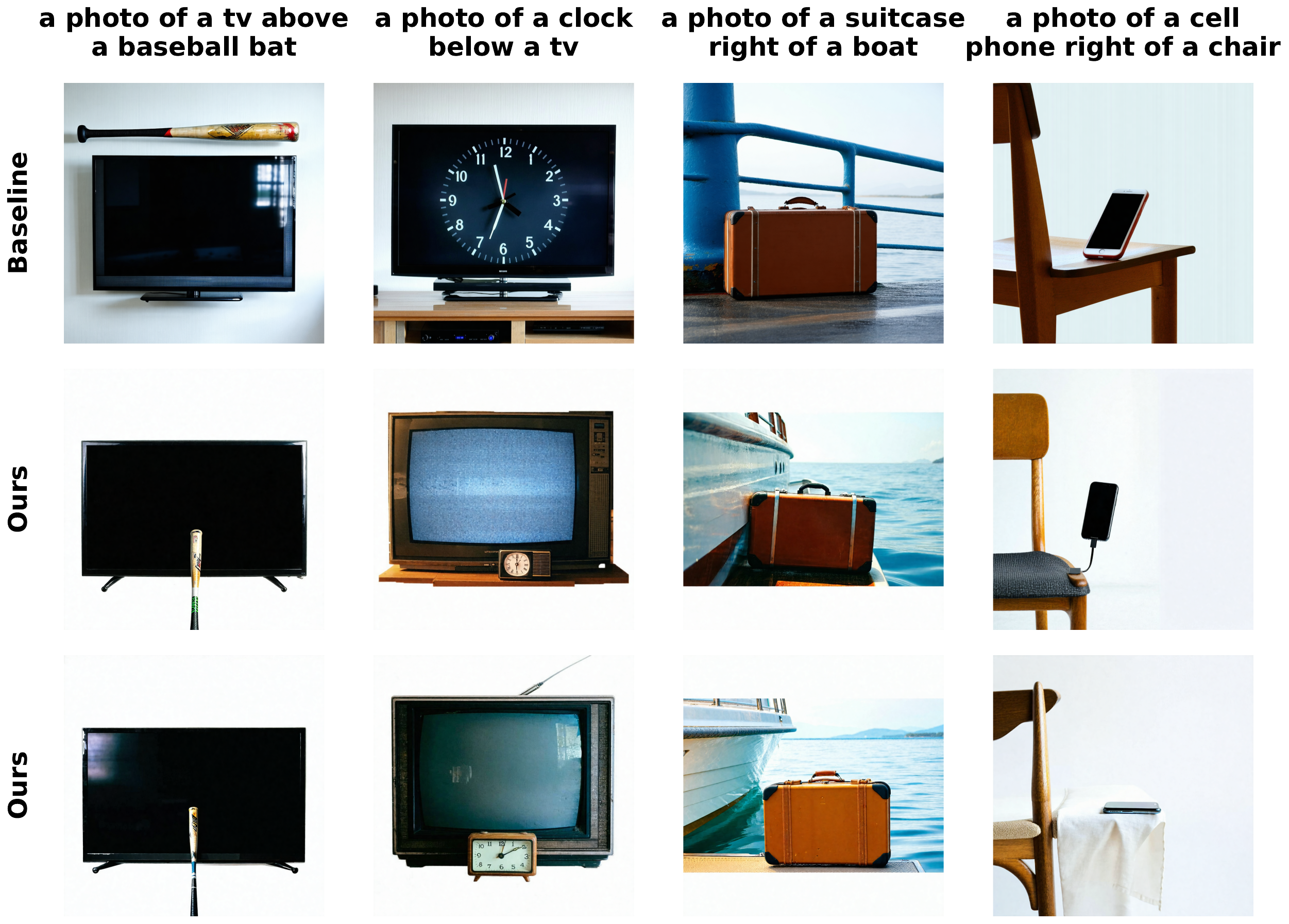}
    \caption{Qualitative comparison highlighting successful corrections of complex spatial relationships.}
    \label{fig:result_position}
\end{figure}

\subsection{Sensitivity Analysis and Ablation Study}
To evaluate the sensitivity of our approach to initial conditions, we assessed the correction performance across different random seeds. 

\begin{table}[H]
  \centering
  \setlength{\tabcolsep}{6pt}
  \caption{Correction success rate using Seed 2025.}
  \label{tab:seed_2025}
  \begin{tabular}{lccc}
    \toprule
    \textbf{Task} & \textbf{ODE Errors} & \textbf{Ours Corrected} & \textbf{Correction Rate (\%)} \\
    \midrule
    Position      & 74  & 25 & 33.78 \\
    Color \& Attr & 39  & 25 & 64.10 \\
    Two Objects   & 9   & 9  & 100.00 \\
    Colors        & 18  & 12 & 66.67 \\
    Counting      & 30  & 26 & 86.67 \\
    \bottomrule
  \end{tabular}
\end{table}

\begin{table}[H]
  \centering
  \setlength{\tabcolsep}{6pt}
  \caption{Correction success rate using Seed 2026.}
  \label{tab:seed_2026}
  \begin{tabular}{lccc}
    \toprule
    \textbf{Task} & \textbf{ODE Errors} & \textbf{Ours Corrected} & \textbf{Correction Rate (\%)} \\
    \midrule
    Position      & 77  & 32 & 41.56 \\
    Color \& Attr & 40  & 28 & 70.00 \\
    Two Objects   & 5   & 5  & 100.00 \\
    Colors        & 13  & 7  & 53.85 \\
    Counting      & 27  & 24 & 88.89 \\
    \bottomrule
  \end{tabular}
\end{table}
As demonstrated in Tables \ref{tab:seed_2025} and \ref{tab:seed_2026}, our method maintains consistent and high correction rates across all evaluated tasks, indicating strong robustness and stability against varying initial noise distributions.

Furthermore, we conducted an ablation study to evaluate the effectiveness and computational cost of our proposed method. Specifically, we compared our weak-to-strong injection approach against a setting that directly uses the primary model for sampling, keeping all other conditions identical. The results demonstrate that our method achieves a 7.14\% improvement in the overall generation score, while only incurring a marginal 3.7\% increase in inference time.

\section{Theoretical Analysis}

This section provides a rigorous mathematical formalization of the continuous-time probability collapse in Flow Matching (FM) models and derives the exact geometric and stochastic interventions (Orthogonal Injection and SDE Fission) required to bypass it.

\subsection{Conditional Expectation Smoothing and the Bias Manifold}

We first formalize the regression objective of Flow Matching models from a measure-theoretic perspective to explain the intrinsic dimensional collapse. Let $\mathcal{X} \subset \mathbb{R}^d$ be the state space, $p_0(x_0) = \mathcal{N}(0, I)$ the prior, $q_{data}(x_1|c)$ the conditional empirical data distribution, and $u_t(x_t|x_0, x_1)$ the target stochastic velocity field.

\begin{lemma}
Under the Minimum Mean Square Error (MMSE) criterion, the optimal parametric vector field $v_\theta^*(x_t, t, c)$ learned by Flow Matching is the conditional expectation of the target velocity field, which acts as a low-pass filter on the high-frequency compositional semantics.[1][2]
\end{lemma}

\begin{proof}
The empirical risk minimization (ERM) objective for FM is defined as :
\begin{equation}
\mathcal{L}_{FM}(\theta) = \mathbb{E}_{t, x_0 \sim p_0, x_1 \sim q_{data}(x_1|c), x_t \sim p_t} \left[ \left\| v_\theta(x_t, t, c) - u_t(x_t | x_0, x_1) \right\|^2 \right]
\end{equation}
By the variational principle and MMSE estimation, the global minimum is exactly the conditional expectation:
\begin{equation}
v_\theta^*(x_t, t, c) = \mathbb{E}_{q(x_0, x_1 | x_t, c)} \left[ u_t(x_t | x_0, x_1) \right]
\end{equation}
For Optimal Transport (OT) paths where $u_t(x_t | x_0, x_1) = x_1 - x_0$, this simplifies to $v_\theta^* = \mathbb{E}_{q(x_1 | x_t, c)}[x_1] - \mathbb{E}_{q(x_0 | x_t, c)}[x_0]$. If $c$ contains a rare composition, $q_{data}(x_1|c)$ is heavily dominated by the empirical prior of the majority class. Consequently, the integral smooths out the high-frequency residual $v_{comp}$, yielding $\lim_{N \to \infty} \int v_{comp}(x_t) q_{data} dx_1 \approx 0$.
\end{proof}

\begin{definition}
Let $\mathcal{M}_t = \{ v_\theta^*(x, t, c) \mid x \in \text{supp}(p_t), c \in \mathcal{C} \}$ denote the image of the smoothed vector field. The effective rank of the covariance matrix of $\mathcal{M}_t$ is strictly less than the theoretical dimension required to express full compositional semantics.
\end{definition}

\begin{theorem}
Given a deterministic Ordinary Differential Equation (ODE) $dx_t = v_\theta(x_t, t, c)dt$ with initial condition $x_0 \sim \mathcal{N}(0, I)$, if the initial velocity vector $v_\theta(x_0, 0, c)$ points into the topological basin of attraction of $\mathcal{M}_t$, the entire generation trajectory remains topologically locked within this low-rank subspace.
\end{theorem}

\begin{proof}
By the Picard-Lindelöf theorem, assuming local Lipschitz continuity of $v_\theta$, the initial value problem yields a unique integral curve. OT-Flow Matching explicitly regularizes the vector field to form straight paths ($\frac{\partial v_\theta}{\partial t} \approx 0$). Thus, the initial tangent $v_\theta(x_0, 0, c)$ globally defines the trajectory's manifold. Standard affine scaling via Classifier-Free Guidance ($\hat{v} = v_{uncond} + \gamma(v_{cond} - v_{uncond})$) strictly operates within $\text{span}(v_{cond}, v_{uncond})$ and cannot recover orthogonal dimensions lost to $\mathcal{N}(v_\theta^*)$.
\end{proof}

\subsection{Geometric Evasion via Orthogonal Semantic Injection}

To escape $\mathcal{M}_t$, we introduce a geometric intervention at the boundary condition $t = t_{start}$. Let $v_T \in \mathbb{R}^D$ be the dominant velocity field of the Teacher model. We utilize a Student probe to calculate an unconditionally-masked velocity $v_S(x_t, t, \tilde{c})$.

\begin{lemma}
The pure compositional semantic residual $v_\perp$, devoid of collinear bias from $v_T$, is obtained via orthogonal projection in the $L^2$ Hilbert space:
\begin{equation}
v_\perp = v_S - \frac{\langle v_S, v_T \rangle}{\| v_T \|^2 + \epsilon} v_T
\end{equation}
By construction, $\langle v_\perp, v_T \rangle \equiv 0$ holds universally.
\end{lemma}

\begin{theorem}
Injecting the orthogonal residual into the initial state update $\hat{v}_{new} = v_T + \alpha v_\perp$ strictly expands the rank of the local tangent space, ejecting the trajectory from $\mathcal{M}_t$.
\end{theorem}

\begin{proof}
The first-order covariance of the original deterministic trajectory is $\Sigma_{old} = v_T v_T^T$, which has $\text{rank}(\Sigma_{old}) = 1$. The modified velocity field yields $\Sigma_{new} = (v_T + \alpha v_\perp)(v_T + \alpha v_\perp)^T$. Since $v_\perp \perp v_T$ (from Lemma 2) and assuming $v_S$ is not strictly collinear with $v_T$ (meaning the student perceives the masked token), $v_T$ and $v_\perp$ are linearly independent. The subspace spanned is $\text{span}(v_T, v_\perp)$, hence:
\begin{equation}
\text{rank}(\Sigma_{new}) = \text{rank}(\Sigma_{old}) + 1
\end{equation}
The integration steps dynamically shift into an unexplored higher-dimensional orthant.
\end{proof}

\subsection{Marginal-Preserving Flow-GRPO SDE Fission}

While Orthogonal Injection unlocks the compositional subspace, deterministically integrating through this highly non-linear, out-of-distribution region leads to discretization errors. We formally construct a Stochastic Differential Equation (SDE) that explores this space via Brownian motion while strictly preserving the original probability flow marginals.

\begin{theorem}
For a continuous-time Flow Matching ODE $dx_t = v_\theta(x_t, t)dt$, an equivalent Itô SDE $dx_t = f_{SDE}(x_t, t)dt + \sigma_t dW_t$ that preserves the identical marginal probability $p_t(x)$ at all $t$ requires the drift term:
\begin{equation}
f_{SDE}(x_t, t) = v_\theta(x_t, t) + \frac{1}{2} \sigma_t^2 \nabla \log p_t(x_t)
\end{equation}
\end{theorem}

\begin{proof}
The marginal distribution $p_t(x)$ under the ODE evolves according to the Continuity Equation:
\begin{equation}
\frac{\partial p_t(x)}{\partial t} = -\nabla \cdot (p_t(x) v_\theta(x, t))
\end{equation}
Conversely, the SDE marginal evolution is governed by the Fokker-Planck (Kolmogorov Forward) equation:
\begin{equation}
\frac{\partial p_t(x)}{\partial t} = -\nabla \cdot (p_t(x) f_{SDE}(x, t)) + \frac{1}{2} \sigma_t^2 \Delta p_t(x)
\end{equation}
To ensure marginal equivalence, we equate the two temporal derivatives:
\begin{equation}
-\nabla \cdot (p_t v_\theta) = -\nabla \cdot (p_t f_{SDE}) + \frac{1}{2} \sigma_t^2 \Delta p_t
\end{equation}
Using the identity $\nabla \cdot (p_t \nabla \log p_t) = \nabla \cdot (\nabla p_t) = \Delta p_t$, we rewrite the Laplacian term:
\begin{equation}
-\nabla \cdot (p_t v_\theta) = -\nabla \cdot \left( p_t \nabla \log p_t \right)
\end{equation}
By eliminating the divergence operators, we obtain the exact closed-form drift:
\begin{equation}
f_{SDE}(x_t, t) = v_\theta(x_t, t) + \frac{1}{2} \sigma_t^2 \nabla \log p_t(x_t)
\end{equation}
\end{proof}

\begin{lemma}
Assuming a Gaussian conditional probability path $p_t(x_t | x_{data}) = \mathcal{N}(x_t; (1-t) x_{data}, t^2 I)$, the required SDE drift can be analytically approximated without retraining a score network.
\end{lemma}

\begin{proof}
For the specified Gaussian path, the score function is [7]:
\begin{equation}
\nabla \log p_t(x_t) = \frac{(1-t)x_{data} - x_t}{t^2}
\end{equation}
Since the optimal Flow Matching vector field points toward the data $v_\theta \approx \frac{x_{data} - x_t}{t}$, substituting $x_{data} \approx x_t + t v_\theta(x_t, t)$ into the score yields:
\begin{equation}
\nabla \log p_t(x_t) \approx \frac{x_t + (1-t)v_\theta(x_t, t)}{t}
\end{equation}
Substituting this into Theorem 4.6 yields the final discrete SDE integration rule used in the scheduler:
\begin{equation}
dx_t = \left[ v_\theta(x_t, t) + \frac{\sigma_t^2}{2t} \left( x_t + (1-t) v_\theta(x_t, t) \right) \right] dt + \sigma_t dW_t
\end{equation}
\end{proof}

\textbf{Conclusion.} The algorithm functions by satisfying Theorem 2 precisely at $t = t_{start}$ to break the Trajectory Lock-in constraint, and subsequently applying the exact SDE conversion derived in Theorem 3 and Lemma 3. With a decaying noise schedule $\sigma_t = a \sqrt{t / (1-t)}$, the Brownian motion safely facilitates high-dimensional exploration and asymptotically collapses back to the deterministic data manifold as $t \to 0$, rigorously preventing generative structure destruction while solving the compositional OOD failure.

\section{Limitations}
\label{sec:limitations}
While InjectFlow demonstrates strong capabilities in mitigating trajectory lock-in, our current study has certain constraints. The empirical evaluation is primarily validated on a single benchmark, GenEval, and relies on a specific weak-strong model pairing (Stable Diffusion 3.5 Large as the primary model and Stable Diffusion 3.5 Medium as the weak auxiliary model). Future work will focus on more extensive validation across diverse compositional datasets to ensure broader applicability. Additionally, we plan to investigate how our orthogonal injection and SDE fission strategies generalize across entirely different foundation model architectures and scaling ratios.

\section{Conclusion}
\label{sec:conclusion}
In this work, we identified and mathematically formalized the ``Bias Manifold'' in continuous-time Flow Matching models, demonstrating how conditional expectation smoothing causes trajectory lock-in and limits minority-class generation. To overcome this, we introduced InjectFlow, a novel, training-free inference intervention. By utilizing a weaker companion model to inject orthogonal semantics during the initial velocity computation, followed by an SDE-driven trajectory fission, our approach successfully forces the generative path out of localized bias subspaces. Our evaluations confirm that InjectFlow significantly repairs complex compositional prompts, providing a ready-to-use solution for enhancing the fairness, diversity, and robustness of visual foundation models.

\clearpage  

%
%
\bibliographystyle{splncs04}
\bibliography{main}
\end{document}